%% file: root.tex
\documentclass[letterpaper, 10 pt, conference]{ieeeconf}  

\IEEEoverridecommandlockouts                              
\usepackage{graphics} 
\usepackage{epsfig} 
\usepackage{mathptmx} 
\usepackage{times} 
\usepackage{amsmath} 
\usepackage{amssymb}  

\usepackage{mathtools}

\usepackage{algorithm}
\usepackage{algorithmic}
\usepackage{xcolor}
\usepackage{hyperref}

\include{math_commands}

\newcommand{\clip}{\text{clip}}

\title{\LARGE \bf
Almost Sure Convergence Analysis of Stochastic Gradient Methods\\ with Clipping and Additive Noise
}

\author{Amartya Mukherjee and Jun Liu
\thanks{This work was supported in part by the Natural Sciences and Engineering Research Council of Canada and the Canada Research Chairs Program.}
\thanks{Amartya Mukherjee and Jun Liu are with the Department of Applied Mathematics, University of Waterloo, Waterloo, Ontario, Canada N2L 3G1 (email: {\tt\small (a29mukhe,j.liu)@uwaterloo.ca}).}%
}

\begin{document}

\maketitle

\begin{abstract}
Stochastic gradient descent (SGD) with gradient clipping and additive noise has become a standard technique for training machine learning models, particularly in applications requiring robustness or privacy guarantees.
However, clipping introduces a bias in stochastic gradients, while additive noise introduces additional variance, making the long-run behaviour of individual optimization trajectories difficult to characterize.
In this work, we prove that SGD with clipping and additive Gaussian noise (SGD-CN) converges almost surely (a.s.) under smoothness and uniformly bounded stochastic-gradient noise assumptions, provided the step sizes satisfy some standard decaying conditions.
Our analysis extends to momentum variants such as the stochastic heavy ball and Nesterov’s accelerated gradient, where we show that careful energy constructions yield similar guarantees.
These results provide stronger theoretical foundations for understanding the pathwise behaviour of clipped stochastic gradient methods and suggest that, despite the bias and noise introduced by clipping and perturbation, the algorithm remains stable in both convex and nonconvex regimes.
\end{abstract}

\section{Introduction}

Gradient clipping is widely used in the training of machine learning models to control large stochastic gradients and improve numerical stability. It also plays an important role in privacy-preserving optimization, where clipping bounds the sensitivity of individual gradient contributions before random noise is added \cite{dwork2006calibrating,abadi2016deep}. These considerations motivate stochastic gradient methods that combine gradient clipping with additive noise. We refer to this class of methods as stochastic gradient descent with clipping and additive noise (SGD-CN).

Despite its practical importance, clipping substantially changes the behaviour of stochastic gradient methods. In particular, even when a stochastic gradient is an unbiased estimator of the true gradient, its clipped version need not remain unbiased. The resulting bias can alter both the magnitude and direction of the expected update. Additive Gaussian perturbations introduce a second difficulty, since individual iterates continue to receive random disturbances throughout training. Correspondingly, standard convergence arguments for unbiased SGD do not directly apply to SGD-CN.

The convergence behaviour of gradient clipping has received significant attention in recent years. Several works analyze clipped stochastic gradient methods under assumptions on the stochastic-gradient noise or on the geometry of the clipping operation. Of particular relevance, \cite{koloskova2023revisiting} characterizes the stochastic bias introduced by gradient clipping and shows that, under a standard bounded-variance assumption alone, clipped SGD need not converge to an exact stationary point. Instead, the clipping bias can produce a nonvanishing neighborhood whose size depends on the stochastic-gradient variance and the clipping threshold. This observation highlights that additional structure on the stochastic-gradient oracle is required for exact convergence with a fixed clipping threshold. Related analyses of gradient clipping and privacy-preserving stochastic optimization include \cite{fang2023improved,tang2024dp}, while differentially private optimization has also been studied in distributed and multi-agent settings \cite{huang2024differential,chen2023differential,katewa2019differential}.

In this work, we study the almost sure convergence of SGD-CN under a uniformly bounded stochastic-gradient noise assumption. Specifically, we assume that the objective is smooth and that the stochastic-gradient noise is uniformly bounded. We show that, although clipping biases the stochastic gradient, its conditional mean remains a descent direction provided that the clipping threshold is sufficiently large relative to the stochastic-gradient noise.
Our convergence analysis builds on this observation. For SGD-CN and its momentum variants, we establish almost sure weighted summability of a clipping-aware stationarity measure, from which best-iterate convergence rates follow. We then strengthen these results to last-iterate convergence by controlling the oscillations of the gradient norm.

\section{Preliminaries and Assumptions}

We provide formal definitions and introduce some assumptions commonly used in the convergence analysis of SGD \cite{liu2022almost,doan2022finite}.

\textbf{Problem statement:} We are interested in solving the following unconstrained minimization problem
\begin{equation}
    \min_{\xv\in\Rb^d}f(\xv),
\end{equation}
where $f:\Rb^d\to\Rb$, using stochastic gradient methods with clipping and additive noise.
Let $f^*$ be the true minimum. In convex settings, we want to prove that $f(\xv_t)-f^*\to 0$ as $t\to\infty$. In nonconvex settings, we want to prove that $\nabla f(\xv_t)\to 0$ as $t\to\infty$.

\begin{definition}[Stochastic Gradient Descent (SGD)]
    The iteration of SGD is given by
    \begin{equation}
        \xv_{t+1}=\xv_t-\alpha_t \gv_t,
    \end{equation}
    where $\gv_t=\nabla f(\xv_t;\xi_t)$ is the stochastic gradient at $\xv_t$ with a random process $\xi_t$ and $\alpha_t$ is a step size. Throughout this paper, we assume the stochastic gradient is unbiased and denote $\Eb_t[\gv_t]=\nabla f(\xv_t)$ as the true gradient.
\end{definition}

\begin{definition}[SGD with Clipping and Additive Noise \cite{abadi2016deep}]\label{def:DP_SGD}
    SGD-CN is a modification of SGD, where gradients are clipped and noise is added to the clipped gradients.
    \begin{equation}\label{eq:DP_SGD}
        \xv_{t+1}=\xv_t-\alpha_t \gv_t^{CN},
    \end{equation}
    where the clipped noisy stochastic gradient $\gv_t^{CN}$ is given by
    \begin{equation}
        \gv_t^{CN}=\clip_q(\nabla f(\xv_t;\xi_t))+q\zeta_t,
    \end{equation}
    where the $\clip_q$ function is defined for $q>0$ by
    \begin{equation}
        \clip_q(\nabla f(\xv_t;\xi_t))=\min\left(1,\frac{q}{\|\nabla f(\xv_t;\xi_t)\|}\right)\nabla f(\xv_t;\xi_t),
    \end{equation}
    where $\zeta_t\sim\Nc(0,\sigma_{CN}^2I)$ is independent zero-mean Gaussian noise with variance $\sigma_{CN}^2$
\end{definition}

\begin{remark}
    SGD-CN can be adapted to other stochastic gradient methods such as stochastic heavy ball and stochastic Nesterov accelerated gradient.
    In the differential privacy (DP) literature, the parameters $q$ and $\sigma_{CN}$ are chosen to satisfy $(\epsilon,\delta)$-DP \cite{dwork2006calibrating}, but our analysis does not rely on this connection.
\end{remark}

We make the following assumptions that are commonly used in the SGD literature \cite{nesterov2004introductory}.

\begin{assumption}[$L$-smoothness]\label{as:L_smoothness}
    $f$ is bounded from below by $f^*:=\min_{\xv\in\Rb^d}f(\xv)$ and its gradient $\nabla f$ is $L$-Lipschitz i.e. 
    \begin{equation}
        \|\nabla f(\xv)-\nabla f(\yv)\|\leq L\|\xv-\yv\|,
    \end{equation}
    for all $\xv,\yv\in\Rb^d$.
\end{assumption}

\begin{assumption}[$\mu$-strongly convex]\label{as:mu_strongly_convex}
    There exists a positive constant $\mu>0$ such that
    \begin{equation}
        f(\yv)\geq f(\xv)+\langle\nabla f(\xv),\yv-\xv\rangle+\frac{\mu}{2}\|\yv-\xv\|^2,
    \end{equation}
    for all $\xv,\yv\in\Rb^d$
    A consequence of $f$ being $\mu$-strongly convex is that
    \begin{equation}\label{eq:strong_convex2}
        \frac{1}{2\mu}\|\nabla f(\xv)\|^2\geq f(\xv)-f^*.
    \end{equation}
\end{assumption}


We impose the following standard bounded-noise assumption on the
stochastic gradient oracle.

\begin{assumption}[Uniformly Bounded Stochastic Gradient Noise]
\label{as:bounded_noise}
There exists a constant $\sigma_g\geq 0$ such that
\begin{equation}
    \left\|
    \nabla f(\xv;\xi)-\nabla f(\xv)
    \right\|
    \leq \sigma_g
    \quad \text{a.s.}
\end{equation}
for every $\xv\in\mathbb{R}^d$.
\end{assumption}

\begin{remark}
    Assumption~\ref{as:bounded_noise} has been extensively studied in \cite{koloskova2023revisiting} and provides a theoretical lower bound on the clipping threshold.
\end{remark}

\input{background_lemmas}

\input{SGD_SHB}

\input{last_iterate}

\section{Conclusion}

We established almost sure convergence guarantees for SGD with gradient clipping and additive Gaussian noise and its SHB and SNAG momentum variants. Under smoothness and uniformly bounded stochastic-gradient noise, $q>\sigma_g$ guarantees positive alignment between the expected clipped gradient and the true gradient. Using this property and supermartingale arguments, we established almost sure best-iterate convergence rates and last-iterate asymptotic stationarity for all three methods. An important direction for future work is to relax the uniformly bounded stochastic-gradient noise assumption, which is stronger than standard bounded-variance conditions. Sharper dependence on the clipping threshold $q$, stochastic-gradient noise $\sigma_g$, and injected noise variance $\sigma_{CN}^2$ would also provide a more quantitative characterization of the trade-off between clipping, noise injection, and optimization performance.

\bibliography{main}
\bibliographystyle{abbrv}

\end{document}

%% file: math_commands.tex
\usepackage{color}

\usepackage{blindtext}
\usepackage{enumerate}
\usepackage{graphicx}
\usepackage{amsfonts, amsmath, bm, amssymb}
\usepackage{dsfont}
\usepackage{wrapfig}
\usepackage{subcaption}

\usepackage{pifont}
\usepackage{xspace}
\makeatletter
\DeclareRobustCommand\onedot{\futurelet\@let@token\@onedot}
\def\@onedot{\ifx\@let@token.\else.\null\fi\xspace}

\makeatother

\newcommand{\Fc}{\mathcal{F}}

\newcommand{\Nc}{\mathcal{N}}

\newcommand{\Eb}{\mathbb{E}}

\newcommand{\Rb}{\mathbb{R}}

\newcommand{\gv}{\mathbf{g}}
\newcommand{\hv}{\mathbf{h}}

\newcommand{\uv}{\mathbf{u}}
\newcommand{\vv}{\mathbf{v}}
\newcommand{\wv}{\mathbf{w}}
\newcommand{\xv}{\mathbf{x}}
\newcommand{\yv}{\mathbf{y}}
\newcommand{\zv}{\mathbf{z}}

\ifx\BlackBox\undefined
\newcommand{\BlackBox}{\rule{1.5ex}{1.5ex}}  
\fi
\ifx\QED\undefined
\def\QED{~\rule[-1pt]{5pt}{5pt}\par\medskip}
\fi
\ifx\proof\undefined
\newenvironment{proof}{\par\noindent{\em Proof:\ }}{\hfill\BlackBox\\}
\fi
\ifx\theorem\undefined
\newtheorem{theorem}{Theorem}
\fi
\ifx\example\undefined
\newtheorem{example}{Example}
\fi
\ifx\property\undefined
\newtheorem{property}{Property}
\fi
\ifx\lemma\undefined
\newtheorem{lemma}{Lemma}
\fi
\ifx\proposition\undefined
\newtheorem{proposition}{Proposition}
\fi
\ifx\fact\undefined
\newtheorem{fact}{Fact}
\fi
\ifx\remark\undefined
\newtheorem{remark}{Remark}
\fi
\ifx\corollary\undefined
\newtheorem{corollary}{Corollary}
\fi
\ifx\definition\undefined
\newtheorem{definition}{Definition}
\fi
\ifx\conjecture\undefined
\newtheorem{conjecture}{Conjecture}
\fi
\ifx\axiom\undefined
\newtheorem{axiom}[theorem]{Axiom}
\fi
\ifx\claim\undefined
\newtheorem{claim}[theorem]{Claim}
\fi
\ifx\assumption\undefined
\newtheorem{assumption}{Assumption}
\fi
\ifx\question\undefined
\newtheorem{question}{Question}
\fi
\ifx\problem\undefined
\newtheorem{problem}{Problem}
\fi

%% file: background_lemmas.tex
\section{Background and Lemmas on Supermartingales}

The analysis in this paper follows from the following result derived in \cite{robbins1971convergence}. Throughout the remainder of the paper, we use the shorthand notation $\Eb_t[\cdot]:=\Eb_t[\cdot|\Fc_t]$, where $\Fc_t$ is the natural filtration.

\begin{proposition}\label{prop:1}
    Let $\{X_t\}$, $\{Y_t\}$, and $\{Z_t\}$ be three sequences of random variables that are adapted to a filtration $\{\mathcal{F}_t\}$. Let $\{\gamma_t\}$ be a sequence of nonnegative real numbers such that $\prod_{t=1}^{\infty} (1 + \gamma_t) < \infty$. Suppose that the following conditions hold:
    \begin{enumerate}
        \item $X_t, Y_t, Z_t$ are nonnegative for all $t \geq 1$.
        \item $\mathbb{E}[Y_{t+1} | \mathcal{F}_t] \leq (1 + \gamma_t) Y_t - X_t + Z_t$ for all $t \geq 1$.
        \item $\sum_{t=1}^{\infty} Z_t < \infty$ holds a.s.
    \end{enumerate}
    Then, we have
    \begin{equation}
        \sum_{t=1}^{\infty} X_t < \infty \quad \text{a.s.,}
    \end{equation}
    and $Y_t$ converges a.s.
\end{proposition}

The following result from \cite{liu2022almost} is used for convergence results in nonconvex settings.

\begin{lemma}[Lemma 2 of \cite{liu2022almost}]\label{lem:jy}
Let $\{X_t\}$ be a sequence of nonnegative real numbers and $\{\alpha_t\}$ be a decreasing sequence of positive real numbers such that the following conditions hold:
\begin{equation}
    \sum_{t=1}^{\infty} \alpha_t X_t < \infty, \quad \sum_{t=1}^{\infty} \frac{\alpha_t}{\sum_{i=1}^{t-1} \alpha_i} = \infty.
\end{equation}
Then,
\begin{equation}
\min_{1 \leq i \leq t} X_i = o \left( \frac{1}{\sum_{i=1}^{t-1} \alpha_i} \right).
\end{equation}
\end{lemma}

We derive some properties of the clipped noisy stochastic gradient that aid our analysis.

\begin{proposition}[Expected Clipped-Gradient Alignment]
\label{prop:clipped_alignment}
    Suppose Assumption~\ref{as:bounded_noise} holds and define
    \begin{equation}
        \Psi_q(r):=\min\{r^2,qr\},
        \quad r\geq 0.
        \label{eq:Psi_def}
    \end{equation}
    Then, for every $\xv\in\mathbb{R}^d$,
    \begin{align}
        &\left\langle
        \nabla f(\xv),
        \Eb_t\left[
            \operatorname{clip}_q(\nabla f(\xv;\xi))
            | \xv
        \right]
        \right\rangle \notag\\
        &\quad\geq
        \|\nabla f(\xv)\|^2
        -
        \|\nabla f(\xv)\|
        \left(
            \|\nabla f(\xv)\|+\sigma_g-q
        \right)_+ .
        \label{eq:general_clipped_alignment}
    \end{align}
    In particular, if $q>\sigma_g$, then
    \begin{equation}
        \left\langle
        \nabla f(\xv),
        \Eb_t\left[
            \operatorname{clip}_q(\nabla f(\xv;\xi))
            | \xv
        \right]
        \right\rangle
        \geq
        \left(1-\frac{\sigma_g}{q}\right)
        \Psi_q\left(\|\nabla f(\xv)\|\right).
        \label{eq:clipped_alignment}
    \end{equation}
\end{proposition}

\begin{proof}
    Fix $\xv\in\mathbb{R}^d$ and write
    \begin{equation*}
        \gv:=\nabla f(\xv),
        \quad
        \gv_\xi:=\nabla f(\xv;\xi),
        \quad
        r:=\|\gv\|.
    \end{equation*}
    Define $\hv_\xi:=\operatorname{clip}_q(\gv_\xi)$.
    By conditional unbiasedness,
    $\Eb_t[\gv_\xi|\xv]=\gv$, and hence
    \begin{align}
        \left\langle
        \gv,
        \Eb_t[\hv_\xi|\xv]
        \right\rangle
        &=
        \left\langle
        \gv,
        \Eb_t[\gv_\xi|\xv]
        \right\rangle
        -
        \left\langle
        \gv,
        \Eb_t[\gv_\xi-\hv_\xi|\xv]
        \right\rangle \notag\\
        &=
        r^2
        -
        \left\langle
        \gv,
        \Eb_t[\gv_\xi-\hv_\xi|\xv]
        \right\rangle \notag\\
        &\geq
        r^2
        -
        r\,\Eb_t[
            \|\gv_\xi-\hv_\xi\|
            |\xv
        ].
        \label{eq:clip_bias_step}
    \end{align}
    By the definition of the clipping operator,
    \begin{equation*}
        \|\gv_\xi-\hv_\xi\|
        =
        \left(\|\gv_\xi\|-q\right)_+.
    \end{equation*}
    Assumption~\ref{as:bounded_noise} further gives
    \begin{equation*}
        \|\gv_\xi\|
        \leq
        \|\gv\|+\|\gv_\xi-\gv\|
        \leq
        r+\sigma_g
        \quad\text{a.s.}
    \end{equation*}
    Therefore,
    \begin{equation*}
        \|\gv_\xi-\hv_\xi\|
        \leq
        (r+\sigma_g-q)_+
        \quad\text{a.s.}
    \end{equation*}
    Substituting this into \eqref{eq:clip_bias_step} yields
    \begin{equation*}
        \left\langle
        \gv,
        \Eb_t[\hv_\xi|\xv]
        \right\rangle
        \geq
        r^2-r(r+\sigma_g-q)_+,
    \end{equation*}
    proving \eqref{eq:general_clipped_alignment}.
    Now suppose $q>\sigma_g$. If $r+\sigma_g\leq q$,
    then
    \begin{equation*}
        \left\langle
        \gv,
        \Eb_t[\hv_\xi|\xv]
        \right\rangle
        \geq r^2
        \geq
        \left(1-\frac{\sigma_g}{q}\right)\Psi_q(r).
    \end{equation*}
    Otherwise, $r+\sigma_g>q$, and
    \begin{equation*}
        r^2-r(r+\sigma_g-q)
        =
        r(q-\sigma_g).
    \end{equation*}
    If $r\leq q$, then
    \begin{equation*}
        r(q-\sigma_g)
        \geq
        \left(1-\frac{\sigma_g}{q}\right)r^2,
    \end{equation*}
    whereas if $r>q$, then
    \begin{equation*}
        r(q-\sigma_g)
        =
        \left(1-\frac{\sigma_g}{q}\right)qr.
    \end{equation*}
    Since $\Psi_q(r)=\min\{r^2,qr\}$, the desired result
    follows.
\end{proof}

\begin{remark}[Choice of clipping threshold]
\label{rem:clipping_threshold}
    Proposition~\ref{prop:clipped_alignment} holds for arbitrary
    $q>0$. However, the condition $q>\sigma_g$ guarantees that
    the expected clipped stochastic gradient is uniformly aligned
    with the true gradient. In particular, if $q\geq 2\sigma_g$, then
    \begin{equation}
        \left\langle
        \nabla f(\xv),
        \Eb_t\left[
            \operatorname{clip}_q(\nabla f(\xv;\xi))
            |\xv
        \right]
        \right\rangle
        \geq
        \frac{1}{2}
        \min\left\{
            \|\nabla f(\xv)\|^2,
            q\|\nabla f(\xv)\|
        \right\}.
    \end{equation}
    More generally, the alignment coefficient is
    $1-\sigma_g/q$, explicitly quantifying the effect of the
    clipping threshold relative to the stochastic-gradient noise.
\end{remark}

%% file: SGD_SHB.tex
\section{Almost Sure Convergence Rate Analysis}
\label{sec:convergence_rates}

In this section, we establish almost sure convergence rates for
SGD-CN and its stochastic heavy-ball extension.
Our analysis is based on the expected clipped-gradient alignment
derived in Proposition~\ref{prop:clipped_alignment}.

Throughout this section, define $\Psi_q(\cdot)$ as in \eqref{eq:Psi_def}.
Under Assumption~\ref{as:bounded_noise}, if $q>\sigma_g$,
Proposition~\ref{prop:clipped_alignment} gives
\begin{equation}
    \left\langle
        \nabla f(\xv),
        \Eb\left[
            \operatorname{clip}_q(\nabla f(\xv;\xi))
            \mid \xv
        \right]
    \right\rangle
    \geq
    \kappa_q
    \Psi_q(\|\nabla f(\xv)\|),
    \label{eq:Psi_alignment}
\end{equation}
where $\kappa_q:=1-\frac{\sigma_g}{q}>0$.
Since $\Eb_t[\zeta_t]=0$, \eqref{eq:Psi_alignment} implies
\begin{equation}
    \Eb_t
    \left[
        \langle\nabla f(\xv_t),\gv_t^{CN}\rangle
    \right]
    \geq
    \kappa_q\Psi_q(\|\nabla f(\xv_t)\|).
    \label{eq:gCN_alignment}
\end{equation}
Moreover,
\begin{align}
    \Eb_t\|\gv_t^{CN}\|^2
    &=
    \Eb_t
    \left\|
        \operatorname{clip}_q(\nabla f(\xv_t;\xi_t))
    \right\|^2
    +
    q^2\Eb_t\|\zeta_t\|^2
    \nonumber\\
    &\leq
    q^2+q^2d\sigma_{CN}^2
    =:Q^2.
    \label{eq:gCN_second_moment}
\end{align}

We consider positive, nonincreasing step sizes satisfying
\begin{equation}
    \alpha_t
    =
    \Theta\left(\frac{1}{t^{1-\theta}}\right),
    \quad
    \theta\in\left(0,\frac{1}{2}\right).
    \label{eq:stepsize_rate}
\end{equation}
In particular,
\begin{equation*}
    \sum_{t=1}^{\infty}\alpha_t=\infty,
    \quad
    \sum_{t=1}^{\infty}\alpha_t^2<\infty,
    \quad
    \sum_{i=1}^{t}\alpha_i=\Theta(t^\theta).
    \label{eq:stepsize_properties}
\end{equation*}

\subsection{Stochastic Gradient Descent}

We prove that, under smoothness and bounded noise assumptions, SGD-CN can achieve almost sure convergence.

\begin{theorem}[Convergence of SGD-CN]
\label{thm:SGD_CN}
    Consider the iterates of SGD-CN in
    \eqref{eq:DP_SGD}.
    Suppose Assumptions~\ref{as:L_smoothness} and
    \ref{as:bounded_noise} hold, let $q>\sigma_g$, and suppose
    that the step sizes satisfy \eqref{eq:stepsize_rate}.
    Then
    \begin{equation}
        \sum_{t=1}^{\infty}
        \alpha_t
        \Psi_q(\|\nabla f(\xv_t)\|)
        <\infty
        \quad\text{a.s.}
        \label{eq:SGD_weighted_sum}
    \end{equation}
    and $f(\xv_t)$ converges almost surely to a finite random
    variable.
    Furthermore,
    \begin{equation}
        \min_{1\leq i\leq t}
        \Psi_q(\|\nabla f(\xv_i)\|)
        =
        o\left(
            \left(\sum_{i=1}^{t-1}\alpha_i\right)^{-1}
        \right)
        =
        o(t^{-\theta})
        \quad\text{a.s.}
        \label{eq:SGD_Psi_rate}
    \end{equation}
    and hence,
    \begin{equation}
        \min_{1\leq i\leq t}
        \|\nabla f(\xv_i)\|
        =
        o(t^{-\theta/2})
        \quad\text{a.s.}
    \end{equation}
    If, in addition, Assumption~\ref{as:mu_strongly_convex}
    holds, then
    \begin{equation}
        \min_{1\leq i\leq t}
        \bigl(f(\xv_i)-f^*\bigr)
        =
        o(t^{-\theta})
        \quad\text{a.s.}
        \label{eq:SGD_SC_rate}
    \end{equation}
\end{theorem}

\begin{proof}
    By $L$-smoothness,
    \begin{align*}
        f(\xv_{t+1})
        \leq\;&
        f(\xv_t)
        -
        \alpha_t
        \langle\nabla f(\xv_t),\gv_t^{CN}\rangle
        +
        \frac{L\alpha_t^2}{2}
        \|\gv_t^{CN}\|^2.
    \end{align*}
    Taking conditional expectation and applying
    \eqref{eq:gCN_alignment} and
    \eqref{eq:gCN_second_moment} yields
    \begin{align*}
        \Eb_t[f(\xv_{t+1})-f^*]
        \leq\;&
        f(\xv_t)-f^*\nonumber\\
        &-
        \kappa_q\alpha_t
        \Psi_q(\|\nabla f(\xv_t)\|)
        +
        \frac{LQ^2}{2}\alpha_t^2.
    \end{align*}

    We apply Proposition~\ref{prop:1} with
    \begin{equation*}
        Y_t:=f(\xv_t)-f^*,
        \quad
        X_t:=
        \kappa_q\alpha_t
        \Psi_q(\|\nabla f(\xv_t)\|),
        \quad
        Z_t:=\frac{LQ^2}{2}\alpha_t^2.
    \end{equation*}
    Since $Y_t,X_t,Z_t\geq0$ and
    $\sum_t\alpha_t^2<\infty$, Proposition~\ref{prop:1}
    implies that $f(\xv_t)-f^*$ converges almost surely and
    \begin{equation*}
        \sum_{t=1}^{\infty}
        \alpha_t
        \Psi_q(\|\nabla f(\xv_t)\|)
        <\infty
        \quad\text{a.s.},
    \end{equation*}
    proving \eqref{eq:SGD_weighted_sum}.
    We next apply Lemma~\ref{lem:jy} pathwise with
    \begin{equation*}
        X_t=\Psi_q(\|\nabla f(\xv_t)\|).
    \end{equation*}
    The step-size conditions required by the lemma hold because
    $\{\alpha_t\}$ is positive and nonincreasing, and
    \begin{equation*}
        \alpha_t
        \sum_{i=1}^{t-1}\alpha_i
        =
        \Theta(t^{-1+2\theta}),
    \end{equation*}
    whose sum diverges for every $\theta>0$.
    Therefore,
    \begin{equation*}
        \min_{1\leq i\leq t}
        \Psi_q(\|\nabla f(\xv_i)\|)
        =
        o\left(
            \left(\sum_{i=1}^{t-1}\alpha_i\right)^{-1}
        \right)
        =
        o(t^{-\theta}),
    \end{equation*}
    almost surely.

    To obtain the gradient-norm rate, define
    \begin{equation*}
        r_t^*
        :=
        \min_{1\leq i\leq t}
        \|\nabla f(\xv_i)\|.
    \end{equation*}
    Since $\Psi_q$ is nondecreasing,
    \begin{equation*}
        \Psi_q(r_t^*)
        =
        \min_{1\leq i\leq t}
        \Psi_q(\|\nabla f(\xv_i)\|)
        \to0.
    \end{equation*}
    Hence $r_t^*\to0$, and for all sufficiently large $t$,
    $r_t^*\leq q$. Therefore
    $\Psi_q(r_t^*)=(r_t^*)^2$, which together with
    \eqref{eq:SGD_Psi_rate} gives
    \begin{equation*}
        r_t^*=o(t^{-\theta/2})
        \quad\text{a.s.}
    \end{equation*}

    Finally, suppose $f$ is $\mu$-strongly convex.
    Then, as a result of Assumption \ref{as:mu_strongly_convex},
    \begin{align*}
        \Psi_q(\|\nabla f(\xv_t)\|)
        &\geq
        \Psi_q\left(
            \sqrt{2\mu(f(\xv_t)-f^*)}
        \right)
        \nonumber\\
        &=
        \min\left\{
            2\mu(f(\xv_t)-f^*),
            q\sqrt{2\mu(f(\xv_t)-f^*)}
        \right\}.
        \label{eq:SC_Psi}
    \end{align*}
    Applying Lemma~\ref{lem:jy} to the right-hand side shows that its best-iterate value is
    $o(t^{-\theta})$ almost surely.
    Since this quantity converges to zero, its first branch
    eventually applies to the best iterate, and hence
    \begin{equation*}
        \min_{1\leq i\leq t}
        (f(\xv_i)-f^*)
        =
        o(t^{-\theta})
        \quad\text{a.s.}
    \end{equation*}
    This completes the proof.
\end{proof}

\subsection{Stochastic Heavy-Ball Method}

The clipped noisy stochastic heavy-ball method is given by
\begin{equation}
    \xv_{t+1}
    =
    \xv_t
    -
    \alpha_t\gv_t^{CN}
    +
    \beta(\xv_t-\xv_{t-1}),
    \quad
    \beta\in[0,1).
    \label{eq:SHBCN_iteration}
\end{equation}
Define
\begin{equation}
    \vv_t:=\xv_t-\xv_{t-1},
    \quad
    \zv_t
    :=
    \xv_t+\frac{\beta}{1-\beta}\vv_t.
    \label{eq:SHB_parameterization}
\end{equation}
Then the SHB-CN iteration can be equivalently written as
\begin{equation}
\begin{split}
\vv_{t+1}
&=
\beta\vv_t-\alpha_t\gv_t^{CN},\\
\zv_{t+1}
&=
\zv_t
-
\alpha_t'\gv_t^{CN},
\quad
\alpha_t'
:=
\frac{\alpha_t}{1-\beta}.
\label{eq:SHB_recursion}
\end{split}
\end{equation}
This transformation is standard in the analysis of stochastic
heavy-ball methods~\cite{liu2024almost}.

\begin{theorem}[Convergence of SHB-CN]
\label{thm:SHB}
    Consider the iterates of SHB-CN in
    \eqref{eq:SHBCN_iteration}, with $\beta\in[0,1)$.
    Suppose Assumptions~\ref{as:L_smoothness} and
    \ref{as:bounded_noise} hold, let $q>\sigma_g$, and suppose
    that the step sizes satisfy \eqref{eq:stepsize_rate}.
    Then
    \begin{align}
        \sum_{t=1}^{\infty}
        \alpha_t
        \Psi_q(\|\nabla f(\xv_t)\|)
        &<\infty,
        \label{eq:SHB_weighted_sum}\quad
        \sum_{t=1}^{\infty}\|\vv_t\|^2
        <\infty
    \end{align}
    almost surely. In particular,
    \begin{equation}
        \|\vv_t\|\to0,
        \quad
        \|\zv_t-\xv_t\|\to0
        \quad\text{a.s.}
        \label{eq:SHB_velocity_zero}
    \end{equation}
    Moreover,
    \begin{equation}
        \min_{1\leq i\leq t}
        \Psi_q(\|\nabla f(\xv_i)\|)
        =
        o(t^{-\theta})
        \quad\text{a.s.},
        \label{eq:SHB_Psi_rate}
    \end{equation}
    and thus,
    \begin{equation}
        \min_{1\leq i\leq t}
        \|\nabla f(\xv_i)\|
        =
        o(t^{-\theta/2})
        \quad\text{a.s.}
        \label{eq:SHB_gradient_rate}
    \end{equation}

    If, in addition, Assumption~\ref{as:mu_strongly_convex}
    holds, then
    \begin{equation}
        \min_{1\leq i\leq t}
        \bigl(f(\xv_i)-f^*\bigr)
        =
        o(t^{-\theta})
        \quad\text{a.s.}
        \label{eq:SHB_SC_rate}
    \end{equation}
\end{theorem}

\begin{proof}
    We first control the momentum variable.
    From \eqref{eq:SHB_recursion},
    \begin{align}
        \Eb_t\|\vv_{t+1}\|^2
        =
        \beta^2\|\vv_t\|^2
        -
        2\alpha_t\beta
        \langle\vv_t,\Eb_t\gv_t^{CN}\rangle
        +
        \alpha_t^2
        \Eb_t\|\gv_t^{CN}\|^2.
        \label{eq:SHB_v_exact}
    \end{align}
    Since the additive Gaussian noise has zero conditional mean,
    \begin{equation*}
        \Eb_t\gv_t^{CN}
        =
        \Eb_t[
            \operatorname{clip}_q(\nabla f(\xv_t;\xi_t))
        ],
    \end{equation*}
    and hence, $\|\Eb_t\gv_t^{CN}\|\leq q$.
    Let $\rho:=\frac{1-\beta^2}{2}>0.$
    Applying Young's inequality to the cross term in
    \eqref{eq:SHB_v_exact} and using
    \eqref{eq:gCN_second_moment} gives
    \begin{align}
        \Eb_t\|\vv_{t+1}\|^2
        &\leq
        (\beta^2+\rho)\|\vv_t\|^2
        +
        \alpha_t^2
        \left(
            Q^2+\frac{\beta^2q^2}{\rho}
        \right)
        \nonumber\\
        &=
        (1-\rho)\|\vv_t\|^2
        +
        C_v\alpha_t^2,
        \label{eq:SHB_v_bound}
    \end{align}
    where $C_v:=Q^2+\frac{\beta^2q^2}{\rho}$.
    We next derive a descent inequality for $f(\zv_t)$.
    By $L$-smoothness and \eqref{eq:SHB_recursion},
    \begin{align}
        \Eb_t[f(\zv_{t+1})]
        \leq\;&
        f(\zv_t)
        -
        \alpha_t'
        \langle
            \nabla f(\zv_t),
            \Eb_t\gv_t^{CN}
        \rangle
        +
        \frac{LQ^2}{2}(\alpha_t')^2.
        \label{eq:SHB_z_descent_1}
    \end{align}
    Decomposing the inner product gives
    \begin{align*}
        \left\langle
            \nabla f(\zv_t),
            \Eb_t\gv_t^{CN}
        \right\rangle
        =\;&
        \left\langle
            \nabla f(\xv_t),
            \Eb_t\gv_t^{CN}
        \right\rangle
        \nonumber\\
        &+
        \left\langle
            \nabla f(\zv_t)-\nabla f(\xv_t),
            \Eb_t\gv_t^{CN}
        \right\rangle.
    \end{align*}
    The first term is bounded using
    \eqref{eq:gCN_alignment}, while the second satisfies
    \begin{align*}
        \left\langle
            \nabla f(\zv_t)-\nabla f(\xv_t),
            \Eb_t\gv_t^{CN}
        \right\rangle
        &\geq
        -
        Lq\|\zv_t-\xv_t\|
        =
        -
        \frac{Lq\beta}{1-\beta}
        \|\vv_t\|.
    \end{align*}
    Hence,
    \begin{equation*}
        \left\langle
            \nabla f(\zv_t),
            \Eb_t\gv_t^{CN}
        \right\rangle
        \geq
        \kappa_q
        \Psi_q(\|\nabla f(\xv_t)\|)
        -
        K\|\vv_t\|,
        \label{eq:SHB_alignment_z}
    \end{equation*}
    where $K:=\frac{Lq\beta}{1-\beta}$.
    Substitution into \eqref{eq:SHB_z_descent_1} yields
    \begin{align}
        \Eb_t[f(\zv_{t+1})-f^*]
        \leq\;&
        f(\zv_t)-f^*
        -
        \kappa_q\alpha_t'
        \Psi_q(\|\nabla f(\xv_t)\|)
        \nonumber\\
        &+
        K\alpha_t'\|\vv_t\|
        +
        \frac{LQ^2}{2}(\alpha_t')^2.
        \label{eq:SHB_z_descent_2}
    \end{align}

    Define the energy
    \begin{equation}
        Y_t
        :=
        f(\zv_t)-f^*
        +
        \|\vv_t\|^2.
        \label{eq:SHB_energy}
    \end{equation}
    Combining \eqref{eq:SHB_v_bound} and
    \eqref{eq:SHB_z_descent_2} gives
    \begin{align*}
        \Eb_t[Y_{t+1}]
        \leq\;&
        Y_t
        -
        \kappa_q\alpha_t'
        \Psi_q(\|\nabla f(\xv_t)\|)
        -
        \rho\|\vv_t\|^2
        \nonumber\\
        &+
        K\alpha_t'\|\vv_t\|
        +
        C_0\alpha_t^2,
        \label{eq:SHB_energy_preYoung}
    \end{align*}
    where
    \begin{equation*}
        C_0
        :=
        C_v
        +
        \frac{LQ^2}{2(1-\beta)^2}.
    \end{equation*}
    Applying Young's inequality once more,
    \begin{equation*}
        K\alpha_t'\|\vv_t\|
        \leq
        \frac{\rho}{2}\|\vv_t\|^2
        +
        \frac{K^2}{2\rho}(\alpha_t')^2.
        \label{eq:SHB_final_Young}
    \end{equation*}
    Thus, for some constant $C>0$ independent of $t$,
    \begin{align}
        \Eb_t[Y_{t+1}]
        \leq\;&
        Y_t
        -
        \kappa_q\alpha_t'
        \Psi_q(\|\nabla f(\xv_t)\|)
        -
        \frac{\rho}{2}\|\vv_t\|^2
        +
        C\alpha_t^2.
        \label{eq:SHB_RS_recursion}
    \end{align}

    We now apply Proposition~\ref{prop:1} with
    \begin{equation}
    \begin{split}
        X_t
        &:=
        \kappa_q\alpha_t'
        \Psi_q(\|\nabla f(\xv_t)\|)
        +
        \frac{\rho}{2}\|\vv_t\|^2,
        \\
        Z_t
        &:=
        C\alpha_t^2.
    \end{split}
    \end{equation}
    Since $\sum_t\alpha_t^2<\infty$, Proposition~\ref{prop:1}
    implies
    \begin{equation*}
    \begin{split}
        \sum_{t=1}^{\infty}
        \alpha_t'
        \Psi_q(\|\nabla f(\xv_t)\|)
        <\infty
        \quad
        \label{eq:SHB_Psi_sum_prime}
        \text{and}
        \quad
        \sum_{t=1}^{\infty}\|\vv_t\|^2
        <\infty
        \quad\text{a.s.}
    \end{split}
    \end{equation*}
    Since $\alpha_t'=\alpha_t/(1-\beta)$,
    the summation relation above is equivalent to
    \eqref{eq:SHB_weighted_sum}.
    Furthermore, $\|\vv_t\|\to0$ a.s.
    And by \eqref{eq:SHB_parameterization},
    \begin{equation*}
        \|\zv_t-\xv_t\|
        =
        \frac{\beta}{1-\beta}\|\vv_t\|
        \to0
        \quad\text{a.s.}
    \end{equation*}

    The rate now follows exactly as in
    Theorem~\ref{thm:SGD_CN}.
    Applying Lemma~\ref{lem:jy} to
    \eqref{eq:SHB_weighted_sum} gives
    \begin{equation*}
        \min_{1\leq i\leq t}
        \Psi_q(\|\nabla f(\xv_i)\|)
        =
        o\left(
            \left(\sum_{i=1}^{t-1}\alpha_i\right)^{-1}
        \right)
        =
        o(t^{-\theta})
        \quad\text{a.s.}
    \end{equation*}
    Since $\Psi_q(r)=r^2$ for $r\leq q$, this implies
    \begin{equation*}
        \min_{1\leq i\leq t}
        \|\nabla f(\xv_i)\|
        =
        o(t^{-\theta/2})
        \quad\text{a.s.}
    \end{equation*}

    If $f$ is additionally $\mu$-strongly convex, then
    \eqref{eq:strong_convex2} and the same argument used in
    Theorem~\ref{thm:SGD_CN} give
    \begin{equation*}
        \min_{1\leq i\leq t}
        (f(\xv_i)-f^*)
        =
        o(t^{-\theta})
        \quad\text{a.s.}
    \end{equation*}
    This completes the proof.
\end{proof}


\subsection{Stochastic Nesterov's Accelerated Gradient}

The iteration of the clipped noisy stochastic Nesterov's accelerated
gradient (SNAG-CN) method is given by
\begin{equation}
\begin{split}
    \yv_{t+1}
    &=
    \xv_t-\alpha_t\gv_t^{CN},\\
    \xv_{t+1}
    &=
    \yv_{t+1}
    +
    \beta(\yv_{t+1}-\yv_t),
\end{split}
\label{eq:SNAGCN_iteration}
\end{equation}
where $\beta\in[0,1)$.

Following \cite{liu2022almost}, define
\begin{equation}
    \vv_t
    :=
    \beta(\yv_t-\yv_{t-1}),
    \quad
    \zv_t
    :=
    \xv_t+\frac{\beta}{1-\beta}\vv_t.
    \label{eq:SNAG_parameterization}
\end{equation}
Then the SNAG-CN iteration can be equivalently written as
\begin{equation}
\begin{split}
    \vv_{t+1}
    &=
    \beta\vv_t
    -
    \beta\alpha_t\gv_t^{CN},\\
    \zv_{t+1}
    &=
    \zv_t
    -
    \alpha_t'\gv_t^{CN},
    \quad
    \alpha_t'
    :=
    \frac{\alpha_t}{1-\beta}.
    \label{eq:SNAG_recursion}
\end{split}
\end{equation}
Thus, the transformed SNAG-CN iteration differs from the
corresponding SHB-CN iteration only through the additional factor
$\beta$ multiplying the stochastic-gradient term in the velocity
recursion.

\begin{theorem}[Convergence of SNAG-CN]
\label{thm:SNAG}
    Consider the iterates of SNAG-CN in
    \eqref{eq:SNAGCN_iteration}, with $\beta\in[0,1)$.
    Suppose Assumptions~\ref{as:L_smoothness} and
    \ref{as:bounded_noise} hold, let $q>\sigma_g$, and suppose
    that the step sizes satisfy \eqref{eq:stepsize_rate}.
    Then
    \begin{align}
        \sum_{t=1}^{\infty}
        \alpha_t
        \Psi_q(\|\nabla f(\xv_t)\|)
        &<\infty,
        \label{eq:SNAG_weighted_sum}\quad
        \sum_{t=1}^{\infty}
        \|\vv_t\|^2
        <\infty
    \end{align}
    almost surely. In particular,
    \begin{equation}
        \|\vv_t\|\to0,
        \quad
        \|\zv_t-\xv_t\|\to0
        \quad\text{a.s.}
        \label{eq:SNAG_velocity_zero}
    \end{equation}
    Moreover,
    \begin{equation}
        \min_{1\leq i\leq t}
        \Psi_q(\|\nabla f(\xv_i)\|)
        =
        o(t^{-\theta})
        \quad\text{a.s.},
        \label{eq:SNAG_Psi_rate}
    \end{equation}
    and thus
    \begin{equation}
        \min_{1\leq i\leq t}
        \|\nabla f(\xv_i)\|
        =
        o(t^{-\theta/2})
        \quad\text{a.s.}
        \label{eq:SNAG_gradient_rate}
    \end{equation}

    If, in addition, Assumption~\ref{as:mu_strongly_convex}
    holds, then
    \begin{equation}
        \min_{1\leq i\leq t}
        \bigl(f(\xv_i)-f^*\bigr)
        =
        o(t^{-\theta})
        \quad\text{a.s.}
        \label{eq:SNAG_SC_rate}
    \end{equation}
\end{theorem}

\begin{proof}
    The proof follows the same argument as
    Theorem~\ref{thm:SHB}. The only difference is the
    velocity recursion \eqref{eq:SNAG_recursion}. In
    particular,
    \begin{align}
        \Eb_t\|\vv_{t+1}\|^2
        &=
        \beta^2
        \Eb_t
        \left\|
            \vv_t-\alpha_t\gv_t^{CN}
        \right\|^2
        \nonumber\\
        &=
        \beta^2\|\vv_t\|^2
        -
        2\beta^2\alpha_t
        \left\langle
            \vv_t,\Eb_t\gv_t^{CN}
        \right\rangle
        +
        \beta^2\alpha_t^2
        \Eb_t\|\gv_t^{CN}\|^2.
        \label{eq:SNAG_velocity_expansion}
    \end{align}
    Recall that
    \begin{equation*}
        \|\Eb_t\gv_t^{CN}\|\leq q,
        \quad
        \Eb_t\|\gv_t^{CN}\|^2\leq Q^2.
    \end{equation*}
    Let
    \begin{equation*}
        \rho:=\frac{1-\beta^2}{2}>0.
    \end{equation*}
    Applying Young's inequality to
    \eqref{eq:SNAG_velocity_expansion} gives
    \begin{align}
        \Eb_t\|\vv_{t+1}\|^2
        &\leq
        (\beta^2+\rho)\|\vv_t\|^2
        +
        \alpha_t^2
        \left(
            \beta^2Q^2
            +
            \frac{\beta^4q^2}{\rho}
        \right)
        \nonumber\\
        &=
        (1-\rho)\|\vv_t\|^2
        +
        C_v^{N}\alpha_t^2,
        \label{eq:SNAG_velocity_bound}
    \end{align}
    where
    \begin{equation*}
        C_v^{N}
        :=
        \beta^2Q^2
        +
        \frac{\beta^4q^2}{\rho}.
    \end{equation*}
    The recursion for $\zv_t$ in
    \eqref{eq:SNAG_recursion} is identical to that of
    SHB-CN. Therefore, the same $L$-smoothness and
    expected clipped-gradient alignment arguments used in
    Theorem~\ref{thm:SHB} yield
    \begin{align*}
        \Eb_t[f(\zv_{t+1})-f^*]
        \leq\;&
        f(\zv_t)-f^*
        -
        \kappa_q\alpha_t'
        \Psi_q(\|\nabla f(\xv_t)\|)
        \nonumber\\
        &+
        K\alpha_t'\|\vv_t\|
        +
        \frac{LQ^2}{2}(\alpha_t')^2,
    \end{align*}
    where $K:=\frac{Lq\beta}{1-\beta}$.
    Combining this inequality with
    \eqref{eq:SNAG_velocity_bound} and defining
    \begin{equation*}
        Y_t
        :=
        f(\zv_t)-f^*+\|\vv_t\|^2,
    \end{equation*}
    gives, after another application of Young's inequality,
    \begin{equation}
        \Eb_t[Y_{t+1}]
        \leq
        Y_t
        -
        \kappa_q\alpha_t'
        \Psi_q(\|\nabla f(\xv_t)\|)
        -
        \frac{\rho}{2}\|\vv_t\|^2
        +
        C_N\alpha_t^2
        \label{eq:SNAG_RS_recursion}
    \end{equation}
    for some constant $C_N>0$ independent of $t$.
    Proposition~\ref{prop:1} therefore implies
    \begin{equation*}
        \sum_{t=1}^{\infty}
        \alpha_t'
        \Psi_q(\|\nabla f(\xv_t)\|)
        <\infty,
        \quad
        \sum_{t=1}^{\infty}\|\vv_t\|^2<\infty
        \quad\text{a.s.}
    \end{equation*}
    Since $\alpha_t'=\alpha_t/(1-\beta)$,
    \eqref{eq:SNAG_weighted_sum} follows, while
    $\sum_t\|\vv_t\|^2<\infty$ implies
    $\vv_t\to0$ and hence
    $\|\zv_t-\xv_t\|\to0$ almost surely.

    The best-iterate and strongly convex rates now follow
    identically to Theorem~\ref{thm:SHB} by applying
    Lemma~\ref{lem:jy}. This proves
    \eqref{eq:SNAG_Psi_rate}--\eqref{eq:SNAG_SC_rate}.
\end{proof}

%% file: last_iterate.tex
\section{Last-Iterate Convergence Analysis}
\label{sec:last_iterate}

The results of Section~\ref{sec:convergence_rates} establish
almost sure weighted summability of the gradient measure
$\Psi_q(r)=\min\{r^2,qr\}$. Weighted summability alone does
not imply last-iterate convergence, since the gradient norm may
oscillate. We use the following result of
\cite{orabona2020almost}, simplified from \cite{bertsekas2000gradient}, to rule out such behaviour.

\begin{lemma}[Lemma 1 of \cite{orabona2020almost}]
\label{lem:orabona}
    Let $\{b_t\}$ and $\{\alpha_t\}$ be nonnegative sequences
    and let $\{\wv_t\}$ be a sequence of vectors.
    Suppose that, for some $p\geq1$,
    \begin{equation}
        \sum_{t=1}^{\infty}\alpha_tb_t^p<\infty,
        \quad
        \sum_{t=1}^{\infty}\alpha_t=\infty.
    \end{equation}
    Furthermore, suppose that there exists $C>0$ such that,
    for every $\tau\geq1$,
    \begin{align}
        |b_{t+\tau}-b_t|
        \leq C\left(
            \sum_{i=t}^{t+\tau-1}\alpha_i b_i
            +
            \left\|
                \sum_{i=t}^{t+\tau-1}\alpha_i\wv_i
            \right\|
        \right),
        \label{eq:orabona_condition}
    \end{align}
    where $\sum_{t=1}^{\infty}\alpha_t\wv_t$ converges.
    Then $b_t\to0$.
    See also Lemma 10 of \cite{liu2024almost} for the case
    $p>0$.
\end{lemma}

\begin{theorem}[Last-iterate convergence]
\label{thm:last}
    Consider the iterates of SGD-CN, SHB-CN and SNAG-CN, with
    $q>\sigma_g$ and $\beta\in[0,1)$.
    Suppose Assumptions~\ref{as:L_smoothness} and
    \ref{as:bounded_noise} hold and
    \begin{equation}
        \sum_{t=1}^{\infty}\alpha_t=\infty,
        \quad
        \sum_{t=1}^{\infty}\alpha_t^2<\infty.
        \label{eq:last_stepsizes}
    \end{equation}
    Then
    \begin{equation}
        \|\nabla f(\xv_t)\|\to0
        \quad\text{a.s.}
    \end{equation}
\end{theorem}

\begin{proof}
    It suffices to prove the result for SHB-CN, since SGD-CN
    is recovered by setting $\beta=0$. Recall
    \begin{equation*}
        \vv_{t+1}
        =
        \beta\vv_t-\alpha_t\gv_t^{CN},
        \quad
        \zv_{t+1}
        =
        \zv_t-\alpha_t'\gv_t^{CN},
        \quad
        \alpha_t'
        :=
        \frac{\alpha_t}{1-\beta},
        \label{eq:last_shb_recursion}
    \end{equation*}
    where
    \begin{equation}
        \zv_t-\xv_t
        =
        \frac{\beta}{1-\beta}\vv_t.
        \label{eq:last_zx}
    \end{equation}
    The supermartingale recursion
    \eqref{eq:SHB_RS_recursion} derived in
    Theorem~\ref{thm:SHB} remains valid under
    \eqref{eq:last_stepsizes}; the particular power-law form of
    the step sizes is only required for the rate result.
    Proposition~\ref{prop:1} therefore gives
    \begin{equation}
        \sum_{t=1}^{\infty}
        \alpha_t'
        \Psi_q(\|\nabla f(\xv_t)\|)
        <\infty,
        \quad
        \sum_{t=1}^{\infty}\|\vv_t\|^2<\infty
        \quad\text{a.s.}
        \label{eq:last_previous}
    \end{equation}
    Moreover, the energy
    \begin{equation*}
        Y_t
        =
        f(\zv_t)-f^*+\|\vv_t\|^2
    \end{equation*}
    converges almost surely.
    We work henceforth on an event of probability one on which
    these properties hold. Since $Y_t$ is bounded,
    $f(\zv_t)-f^*$ is bounded. By $L$-smoothness and the lower
    boundedness of $f$,
    \begin{equation*}
        \|\nabla f(\zv_t)\|^2
        \leq
        2L\bigl(f(\zv_t)-f^*\bigr),
    \end{equation*}
    so $\{\|\nabla f(\zv_t)\|\}$ is bounded.
    Since \eqref{eq:last_previous} implies $\vv_t\to0$,
    \eqref{eq:last_zx} and $L$-smoothness show that
    $\{\|\nabla f(\xv_t)\|\}$ is also bounded.
    Thus, for some finite $M>0$,
    \begin{equation*}
        \Psi_q(\|\nabla f(\xv_t)\|)
        \geq
        \min\left\{1,\frac{q}{M}\right\}
        \|\nabla f(\xv_t)\|^2.
    \end{equation*}
    Hence \eqref{eq:last_previous} gives
    \begin{equation}
        \sum_{t=1}^{\infty}
        \alpha_t'
        \|\nabla f(\xv_t)\|^2
        <\infty.
        \label{eq:last_x_weighted}
    \end{equation}
    Using \eqref{eq:last_zx},
    \begin{align*}
        \|\nabla f(\zv_t)\|^2
        &\leq
        2\|\nabla f(\xv_t)\|^2
        +
        2L^2\|\zv_t-\xv_t\|^2
        \nonumber\\
        &=
        2\|\nabla f(\xv_t)\|^2
        +
        \frac{2L^2\beta^2}{(1-\beta)^2}\|\vv_t\|^2.
    \end{align*}
    Since $\{\alpha_t'\}$ is bounded,
    \eqref{eq:last_previous} and \eqref{eq:last_x_weighted}
    imply
    \begin{equation}
        \sum_{t=1}^{\infty}
        \alpha_t'\|\nabla f(\zv_t)\|^2
        <\infty.
        \label{eq:last_z_weighted}
    \end{equation}
    It remains to verify the oscillation condition in
    Lemma~\ref{lem:orabona}. Define
    \begin{equation*}
        \hv_t
        :=
        \clip_q(\nabla f(\xv_t;\xi_t)),
        \quad
        \bar{\hv}_t
        :=
        \Eb_t[\hv_t],
        \quad
        \uv_t
        :=
        \hv_t-\bar{\hv}_t.
        \label{eq:last_decomp}
    \end{equation*}
    We first note that
    \begin{equation}
        \|\bar{\hv}_t\|
        \leq
        c_q\|\nabla f(\xv_t)\|,
        \quad
        c_q:=\frac{q}{q-\sigma_g}.
        \label{eq:last_mean_clip_bound}
    \end{equation}
    Indeed, if $\|\nabla f(\xv_t)\|\leq q-\sigma_g$, then
    Assumption~\ref{as:bounded_noise} implies that clipping is
    inactive almost surely and
    $\bar{\hv}_t=\nabla f(\xv_t)$. Otherwise,
    $\|\bar{\hv}_t\|\leq q
    \leq c_q\|\nabla f(\xv_t)\|$.
    Let
    \begin{equation*}
        b_t:=\|\nabla f(\zv_t)\|,
        \quad
        c_\beta:=\frac{\beta}{1-\beta}.
    \end{equation*}
    By $L$-smoothness and \eqref{eq:last_zx},
    \begin{equation*}
        \|\bar{\hv}_t\|
        \leq
        c_q b_t+c_qLc_\beta\|\vv_t\|.
        \label{eq:last_mean_z}
    \end{equation*}
    Moreover,
    \begin{equation*}
        \gv_t^{CN}
        =
        \bar{\hv}_t+\uv_t+q\zeta_t.
    \end{equation*}
    Therefore, for every $\tau\geq1$,
    \begin{align}
        |b_{t+\tau}-b_t|
        &\leq
        L\|\zv_{t+\tau}-\zv_t\|
        \nonumber\\
        &\leq
        Lc_q
        \sum_{i=t}^{t+\tau-1}\alpha_i'b_i
        +
        Lc_qLc_\beta
        \sum_{i=t}^{t+\tau-1}\alpha_i'\|\vv_i\|
        \nonumber\\
        &\quad+
        L\left\|
            \sum_{i=t}^{t+\tau-1}
            \alpha_i'(\uv_i+q\zeta_i)
        \right\|.
        \label{eq:last_oscillation}
    \end{align}
    Now define the augmented random vector
    \begin{equation}
        \wv_t
        :=
        \begin{bmatrix}
            \uv_t+q\zeta_t\\
            c_qLc_\beta\|\vv_t\|
        \end{bmatrix}.
        \label{eq:last_augmented_noise}
    \end{equation}
    Since $\uv_t$ is a martingale difference and
    $\Eb_t\|\uv_t\|^2\leq q^2$, while $\zeta_t$ is an
    independent zero-mean Gaussian vector,
    \begin{equation*}
        \sum_{t=1}^{\infty}
        \alpha_t'(\uv_t+q\zeta_t)
    \end{equation*}
    converges almost surely by the martingale convergence
    theorem \cite{williams1991probability}, since $\sum_t(\alpha_t')^2<\infty$.
    Furthermore, by Cauchy--Schwarz and
    \eqref{eq:last_previous},
    \begin{align*}
        \sum_{t=1}^{\infty}\alpha_t'\|\vv_t\|
        &\leq
        \left(
            \sum_{t=1}^{\infty}(\alpha_t')^2
        \right)^{1/2}
        \left(
            \sum_{t=1}^{\infty}\|\vv_t\|^2
        \right)^{1/2}
        <\infty.
        \label{eq:last_v_abs}
    \end{align*}
    Hence
    \begin{equation*}
        \sum_{t=1}^{\infty}\alpha_t'\wv_t
    \end{equation*}
    converges almost surely.
    Finally, using
    $a+b\leq\sqrt{2(a^2+b^2)}$ in
    \eqref{eq:last_oscillation}, there exists a constant
    $C>0$ such that
    \begin{equation}
        |b_{t+\tau}-b_t|
        \leq
        C\left(
            \sum_{i=t}^{t+\tau-1}\alpha_i'b_i
            +
            \left\|
                \sum_{i=t}^{t+\tau-1}\alpha_i'\wv_i
            \right\|
        \right).
        \label{eq:last_orabona}
    \end{equation}
    Together with \eqref{eq:last_z_weighted},
    $\sum_t\alpha_t'=\infty$, and
    Lemma~\ref{lem:orabona} with $p=2$, this yields
    \begin{equation*}
        \|\nabla f(\zv_t)\|\to0
        \quad\text{a.s.}
    \end{equation*}
    Finally, $\vv_t\to0$ and \eqref{eq:last_zx} imply
    $\|\zv_t-\xv_t\|\to0$, and therefore
    \begin{equation*}
        \|\nabla f(\xv_t)\|
        \leq
        \|\nabla f(\zv_t)\|
        +
        L\|\xv_t-\zv_t\|
        \to0
        \quad\text{a.s.}
    \end{equation*}
    Setting $\beta=0$ recovers the SGD-CN result. The proof of convergence for SNAG-CN is similar, following \eqref{eq:SNAG_velocity_expansion}.
\end{proof}

%% file: main.bib
@inproceedings{liu2022almost,
  title={On almost sure convergence rates of stochastic gradient methods},
  author={Liu, Jun and Yuan, Ye},
  booktitle={Conference on Learning Theory},
  year={2022},
}

@article{nesterov2004introductory,
  title={{Introductory Lectures on Convex Optimization}},
  author={Nesterov, Yurii},
  journal={Applied Optimization},
  volume={87},
  year={2004},
  publisher={Springer US}
}

@incollection{robbins1971convergence,
  title={A convergence theorem for non negative almost supermartingales and some applications},
  author={Robbins, Herbert and Siegmund, David},
  booktitle={Optimizing methods in statistics},
  pages={233--257},
  year={1971},
  publisher={Elsevier}
}

@inproceedings{dwork2006calibrating,
  title={Calibrating noise to sensitivity in private data analysis},
  author={Dwork, Cynthia and McSherry, Frank and Nissim, Kobbi and Smith, Adam},
  booktitle={Theory of Cryptography: Third Theory of Cryptography Conference},
  pages={265--284},
  year={2006},
  organization={Springer}
}

@inproceedings{abadi2016deep,
  title={Deep learning with differential privacy},
  author={Abadi, Martin and Chu, Andy and Goodfellow, Ian and McMahan, H Brendan and Mironov, Ilya and Talwar, Kunal and Zhang, Li},
  booktitle={Proceedings of the 2016 ACM SIGSAC Conference on Computer and Communications Security},
  pages={308--318},
  year={2016}
}

@article{orabona2020almost,
  title={Almost sure convergence of {SGD} on smooth nonconvex functions},
  author={Orabona, Francesco},
  journal={Blogpost at https://parameterfree.com/2020/10/05/almost-sure-convergence-of-sgd-on-smooth-non-convex-functions},
  year={2020}
}

@book{williams1991probability,
  title={{Probability with Martingales}},
  author={Williams, David},
  year={1991},
  publisher={Cambridge University Press}
}

@article{liu2024almost,
  title={Almost sure convergence rates analysis and saddle avoidance of stochastic gradient methods},
  author={Liu, Jun and Yuan, Ye},
  journal={Journal of Machine Learning Research},
  volume={25},
  number={271},
  pages={1--40},
  year={2024}
}

@inproceedings{fang2023improved,
  title={Improved convergence of differential private sgd with gradient clipping},
  author={Fang, Huang and Li, Xiaoyun and Fan, Chenglin and Li, Ping},
  booktitle={The Eleventh International Conference on Learning Representations},
  year={2023}
}

@inproceedings{koloskova2023revisiting,
  title={Revisiting gradient clipping: Stochastic bias and tight convergence guarantees},
  author={Koloskova, Anastasia and Hendrikx, Hadrien and Stich, Sebastian U},
  booktitle={International Conference on Machine Learning},
  year={2023},
}

@inproceedings{tang2024dp,
  title={{DP-AdamBC}: Your {DP-Adam} is actually {DP-SGD} (unless you apply bias correction)},
  author={Tang, Qiaoyue and Shpilevskiy, Frederick and L{\'e}cuyer, Mathias},
  booktitle={Proceedings of the AAAI Conference on Artificial Intelligence},
  year={2024}
}

@article{doan2022finite,
  title={Finite-time analysis of markov gradient descent},
  author={Doan, Thinh T},
  journal={IEEE Transactions on Automatic Control},
  volume={68},
  number={4},
  pages={2140--2153},
  year={2022},
  publisher={IEEE}
}

@article{huang2024differential,
  title={Differential privacy in distributed optimization with gradient tracking},
  author={Huang, Lingying and Wu, Junfeng and Shi, Dawei and Dey, Subhrakanti and Shi, Ling},
  journal={IEEE Transactions on Automatic Control},
  volume={69},
  number={9},
  pages={5727--5742},
  year={2024},
  publisher={IEEE}
}

@article{chen2023differential,
  title={Differential privacy for symbolic systems with application to Markov Chains},
  author={Chen, Bo and Leahy, Kevin and Jones, Austin and Hale, Matthew},
  journal={Automatica},
  volume={152},
  pages={110908},
  year={2023},
  publisher={Elsevier}
}

@article{katewa2019differential,
  title={Differential privacy for network identification},
  author={Katewa, Vaibhav and Chakrabortty, Aranya and Gupta, Vijay},
  journal={IEEE Transactions on Control of Network Systems},
  volume={7},
  number={1},
  pages={266--277},
  year={2019},
  publisher={IEEE}
}

@article{bertsekas2000gradient,
  title={Gradient convergence in gradient methods with errors},
  author={Bertsekas, Dimitri P and Tsitsiklis, John N},
  journal={SIAM Journal on Optimization},
  volume={10},
  number={3},
  pages={627--642},
  year={2000},
  publisher={SIAM}
}
